\RequirePackage{fix-cm}
\documentclass[smallextended]{svjour3r}       % onecolumn (second format)
\smartqed  % flush right qed marks, e.g. at end of proof
\usepackage{color}
\usepackage{cmll}
\usepackage{graphicx}
\usepackage{epstopdf}
\usepackage{mathtext}
\usepackage{amssymb,amsmath,latexsym,amsfonts,array,epsfig}
\usepackage{mathrsfs}
\usepackage[matrix,arrow,curve]{xy}
\usepackage{xypic}
\begin{document}

\title{Game Semantics and Linear Logic in the Cognition Process \thanks{The project was partly supported by RFBR grant 16-08-00832a}
}
%\subtitle{Do you have a subtitle?\\ If so, write it here}

%\titlerunning{Short form of title}        % if too long for running head

\author{Dmitry Maximov
}

%\authorrunning{Short form of author list} % if too long for running head

\institute{Dmitry Maximov \at
              Trapeznikov Institute of Control Science Russian Academy of Sciences, 65,
              Profsoyuznaya st., Moscow, Russia, 117997 \\
              ORCID: 0000-0001-8610-4342\\
              Tel.: +7-495-3348721\\
              \email{dmmax@inbox.ru}
}

%\date{Received: date / Accepted: date}
% The correct dates will be entered by the editor

\maketitle

\begin{abstract}
A description of the environment cognition process by intelligent systems with the fixed set of system goals is suggested. Such a system is represented by the set of its goals only without any models of the system elements or the environment. The set has a lattice structure and a monoid structure; thus, the structure of linear logic is defined on the set. The cognition process of some environment by the system is described on this basis. The environment is represented as a configuration space of possible system positions which are estimated by an information amount (by corresponding sets). This information is supplied to the system by the environment. Thus, it is possible to define the category of Conway games with a payoff on the configuration space and to choose an optimal system's play (i.e., a trajectory). The choice is determined by the requirement of maximal information increasing and takes into account the structure of the system goal set: the linear logic on the set is used to determine the priority of possible different parallel processes.
The survey may be useful to describe the behavior of robots and simple biological systems, e.g., ants.
% \PACS{PACS code1 \and PACS code2 \and more}
% \subclass{MSC code1 \and MSC code2 \and more}
\end{abstract}

\section{Introduction}
\label{intro}
Universal artificial intelligence (UAI) is a unifying framework and general, formal foundational
theory for artificial intelligence investigations \cite{Hutt18}. Its primary goal is to give a mathematical answer to the question: what is the right thing to do in unknown environments? Investigations in the field are focused on systems which \emph{act} rationally. The artificial intelligence is represented in the approach as an information processor which consumes and gives out information and the theory also tries to answer in general the question ``how can a system composed of relatively unintelligent parts (say, neurons or transistors) behave intelligently?'' \cite{Cass}.

A formal description of the most intelligent agent behavior, in the sense of some intelligence measure, is suggested in UAI framework \cite{HuttB}, \cite{Hutt}. The framework specifies how an agent \emph{interacts} with an environment. The model is based on probabilistic modeling of the environment, on the next system move determination based on previous experience, on a numerical estimation of the system position reward and the maximization of the expected reward along the trajectory. However, the method to obtain this numerical estimation is absent.

It has been demonstrated in \cite{Maximov_17}, \cite{Maximov}, \cite{UBS}, that the structure existence (a lattice structure or else a monoid structure, i.e., the linear logic structure) in the system task \cite{Maximov_17}, \cite{Maximov}, or goal \cite{UBS} set is sufficient for the system to behave quite reasonable. The behavior looks even like ants' behavior in something \cite{Maximov_17}, thus, we may suppose that a system intelligence is the consequence of the system purpose or predestination in teleological spirit. The approach does not suppose the environment modeling unlike \cite{HuttB}.

In this paper, the topic is developed based on the idea that it is possible to represent the cognition process of an environment as fulfilling of parallel achievement processes of different goals in the environment. A tensor multiplication in linear logic corresponds to these parallel processes. The logic is modeled in some game category \cite{resource}. Thus, it is possible to describe goals achievement process by the intelligent system in some environment as a game. Position rewards in the game are represented by sets which define the information about goals or their distinctness degree. Thus, the environment provides the rewards which are lattice elements but not numbers as in \cite{HuttB}. The environment provides these lattice and linear logic structures. But similar structures may also exist on the system goal set, and our ideas about the system purposes give them. These are used to determine the priority of different parallel processes in such game category.

The paper is organized as follows. Section \ref{sec:1} presents the mathematical backgrounds necessary for the understanding of the results. Two propositions are also proved here which adapt the theory of \cite{resource} to the paper purposes. Section \ref{sec:3} represents a formal description of intelligent system behavior and the results discussion. In Sec.\ref{concl}, we conclude the paper.

\section{Mathematical Backgrounds}
\label{sec:1}
\newtheorem{Th}{Supposition}
\subsection{Lattices}
\label{subsec:1}
\begin{definition}
A \textbf{partially-ordered set} $P$ is the set with such a binary relation  $x \leqslant y$ for elements in it, that for all $x, y, z \in P$ the next relationships are performed:
\begin{itemize}
\item	$x \leqslant x$ (reflexivity);
\item	if $x \leqslant y$ and $y \leqslant x$, then $x = y$ (anti-symmetry);
\item	if $x \leqslant y$ and $y \leqslant z$, then $x \leqslant z$ (transitivity).
\end{itemize}
\end{definition}

The definition means that in the partially-ordered set not all elements are compared with each other. This property distinguishes these sets from linear-ordered ones, i.e., from numeric sets which are ordered by a norm.  Thus, the elements of the partially-ordered set are the objects of more general nature than numbers. In the partially-ordered set diagram, the greater the element (i.e., vertex, node) the higher it lies, and the elements are compared with each other lie in the same path from a bottom element to a top one. An example of a partially-ordered set diagram is represented in Fig.~\ref{fig:1} which is also a lattice diagram.
\begin{figure}
% Use the relevant command to insert your figure file.
% For example, with the graphicx package use
  \includegraphics[scale=0.7]{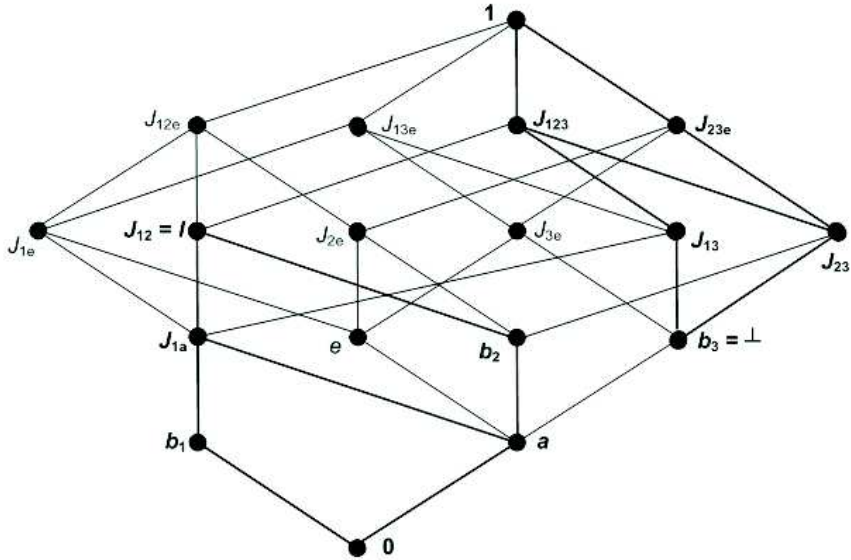}
% figure caption is below the figure
\caption{A lattice example}
\label{fig:1}       % Give a unique label
\end{figure}
\begin{definition}The \textbf{upper bound} of a subset $X$ in a partially-ordered set $P$ is the element $a\in P$, containing all $x\in X$.
\end{definition}

The supremum or \emph{join} is the smallest subset $X$ upper  bound. The infimum or \emph{meet} defines dually as the greatest element $a\in P$ containing in all $x\in X$.
\begin{definition} A \textbf{lattice} is a partially-ordered set, in which every two elements have their meet, denoting as $x\wedge y$, and join, denoting as $x\vee y$.
\end{definition}

In the lattice diagram the elements join is the nearest upper element to both of them, and the meet is the nearest lower one to both. The elements generating by joins and meets all other elements are called \textbf{generators}.

They refer to the lattice as \emph{complete} lattice if its arbitrary subset has the join and the meet. Thus, any complete lattice has the greatest element 1, and the smallest one 0 and every finite lattice is complete \cite{Birkhoff}.

\subsection{Linear Logic}
\label{subsec:2}
If a multiplication operation is additionally defined at the lattice elements, then the operations of linear logic also exist at the lattice. We use the phase semantic of linear logic from \cite{jerar1987}.
\begin{definition}
A \textbf{phase space} is a pare $(M,\bot)$, where $M$ is a multiplicative monoid (i.e., a triple $(M_{0}, \cdot, e)$ with $M_{0}$ is a set and $\cdot$ is a muliplication with the unit $e$), which is also a lattice, and the element \emph{false} of the lattice $\bot\subset M$ is an arbitrary subset of the monoid.
\end{definition}

In linear logic, the element \emph{false} differs from 0 (the minimal lattice element) in general in contrast to classical logic or intuitionistic one. The multiplication $X\cdot Y = \{x\cdot y|x\in X, y\in Y\}$ is defined for arbitrary monoid subsets (i.e. the lattice elements) $X, Y \subset M$. The linear implication $X\multimap Y = \{z|x\cdot z\in Y, \forall x\in X\}$ is also defined. For $X\subset M$ its dual is defined as $X^{\bot} = X\multimap \bot$. The dual element is a generalization of the negation in the case of linear logic.
\begin{definition}
\textbf{Facts} are such subsets $X\subset M$ that $X^{\bot\bot} = X$ or equivalently $X = Y^{\bot}$ for some $Y \subset M$.
\end{definition}

Thus, facts are lattice elements coinciding with their double negations. E.g. $\bot^{\bot} = I = \{e\}^{\bot\bot};\; 1 = M = \emptyset^{\bot};\; 0 = 1^{\bot} = M^{\bot} = \emptyset^{\bot\bot}$. Here 1 is the maximal element of the lattice $M$, 0 is its minimal element, $e$ is the monoidal unit, $I$ is the neutral element of the multiplicative conjunction (see after this).

It is easy to get the next properties: $X^{\bot}X\subset \bot;\; X\subset X^{\bot\bot};\; X^{\bot\bot\bot} = X;\; X\multimap Y^{\bot} = (X\cdot Y)^{\bot};\; (X\vee Y)^{\bot} = X^{\bot}\wedge Y^{\bot}$. From here we get only facts may be the values and the consequents of the implication.

At facts the lattice operations of the additive conjunction $\&$ and the additive disjunction + are defined in the next way: $X \& Y = X\wedge Y = (X^{\bot}\vee Y^{\bot})^{\bot};\; X + Y = (X^{\bot} \& Y^{\bot})^{\bot} = (X^{\bot}\wedge Y^{\bot})^{\bot} = (X \vee Y)^{\bot\bot}$. The duality of the operations is understood here as in the set theory: $\vee^{\bot} = \wedge;\; \wedge^{\bot} = \vee$ in which the duality means the negation.

At facts, multiplicative operations are also defined. Those are the multiplicative conjunction $\times$ and the multiplicative disjunction $\parr$: $X\times Y = (X\cdot Y)^{\bot\bot} = (X\multimap Y^{\bot})^{\bot} = (X^{\bot}\parr Y^{\bot})^{\bot};\; X\parr Y = (X^{\bot}\cdot Y^{\bot})^{\bot} =  X^{\bot}\multimap Y$. The neutral element of the operation $\&$ is 1, the dual to it (neutral element of the operation +) is 0. The neutral element of the operation $\parr$ is $\bot$, the dual to it, the neutral element of the operation $\times$, is $I$.

The set of facts is divided into two classes dual to each other: the class of \textbf{open} facts $\emph{Op}$ and the class of \textbf{closed} facts $\emph{Cl}$. The set $\emph{Op}$ is closed by operations + and $\times$. Its maximal element by inclusion is $I$, and the minimal one is 0. The set $\emph{Cl}$ is closed correspondingly by operations $\&$ and $\parr$, and its maximal element is 1, and the minimal one is $\bot$. The example of such a choice of $\bot$ and $I$ elements, and the open and close fact classes are depicted in Fig \ref{fig:1} by bold lines and symbols:

$J_{12} = I = b^{\bot}_{3} = {\bot}^{\bot}$;

$J_{1a} = J^{\bot}_{23};\;b_{1} = J^{\bot}_{23e}$;

$b_{2} = J^{\bot}_{13};\;a = J^{\bot}_{123};\;0 = {\top}^{\bot}$.

\noindent Duals for non-facts are chosen in \cite{Maximov_18} in the next way:

%\begin{gather}\label{0}
$J^{\bot}_{12e} = J^{\bot}_{13e} = J^{\bot}_{1e} = 0$

%\end{gather}
%\begin{gather}\label{x1}
$J^{\bot}_{2e} = J^{\bot}_{3e} = e^{\bot} = b_{1}$
%\end{gather}

\noindent These definitions are used in Sec. \ref{subsec:3.1}.

The intuition of linear logic came from the fact that the premise of the implication may be used in an inference only once, i.e., the premise $a$ in $a\multimap b$ is treated as a resource which is consumed in the $b$ receiving process. For example, we may spend 1\$ to buy one pack of Camel or Marlborough, but not both of them that is represented by the formula $1\$\multimap 1Camel \& 1Marlborough$. In linear logic the formula $a\multimap b = b^{\bot}\multimap a^{\bot}$ is valid. Thus, the  term $a^{\bot}$ is treated as the resource absence, i.e., the above formula means that $a$ consuming and $b$ getting is the same as $b$ absence consuming (i.e., $b$ arising) and $a$ absence getting (i.e., $a$ annihilating). In this intuition, the tensor product $\times$ is considered as parallel processes fulfilling, i.e., both the resources $a$ and $b$ should be spent in parallel to receive $c$ in the formula $a\times b\multimap c$ unlike of the case $a\multimap b\& c$ when we should choose to what process to spend \cite{jerar1987}\cite{llp}.

%Thus, the formula $a\multimap b\times c$ is not derivable in linear logic, though, e.g., $a\times b\multimap c$ is derivable .

\subsection{Game Semantics}
\label{subsec:3}
\begin{definition} \cite{resource}
A \textbf{Conway game} is defined as rooted graph with vertices \emph{V} as the game positions and edges $E\subset V\times V$  as the game moves. Each edge has a \textbf{polarity} $\pm 1$  which depends on whether it is the Proponent or the Opponent move.
\end{definition}
\begin{definition} \cite{resource}
A trajectory or a \textbf{play} is some path from the graph root $\ast$. The path is \textbf{alternated} if the adjacent edges are of different polarities.
\end{definition}
\begin{definition} \cite{resource}
A \textbf{strategy} $\sigma$ of a Conway game is defined as a non-empty set of alternated plays (paths) of even length, which are started from the Opponent move, closed up to the prefix of even length, i.e., for all plays \emph{s} and all moves \emph{m}, \emph{n},
$s \cdot m \cdot n \in \sigma$ implies $s \in \sigma$, and detemined.  Determinism means that two different paths with common prefix should coincide, i.e., for all plays \emph{s}, and for all moves \emph{m}, \emph{n}, and \emph{n'},
$s \cdot m \cdot n \in \sigma$ and $s \cdot m \cdot n' \in \sigma$ implies $n = n'$.
\end{definition}
\begin{definition} \cite{resource}
A dual play $X^{\bot}$ is obtained from the play \emph{X} by reversing the
polarity of moves.
\end{definition}
\begin{definition} \cite{resource}
The tensor product $X\otimes Y$ of two Conway games \emph{X} and \emph{Y} is the product of the two underlying graphs, i.e. positions $x\otimes y$ are $V_{X\otimes Y}=V_{X}\times V_{Y}$ with the root $\ast_{X\otimes Y}=\ast\times \ast_{Y}$, moves are $x\otimes y\rightarrow \left\{\begin{aligned}z\otimes y; x\rightarrow z \;in \;X\\
                      x\otimes z; y\rightarrow z \;in \;X\end{aligned}
\right.$ and the polarity of a move in $X\otimes Y$ is inherited from the polarity of the
underlying move in \emph{X} or \emph{Y}.
\end{definition}
Generalized linear logic is modeled in the category \textbf{Conw} of such games \cite{resource}. The category objects are Conway games and morphisms $X\rightarrow Y$  are strategies in  $X^{\bot}\parr Y$. The morphism composition and the identity morphism are apparent \cite{resource}.
\begin{definition} \cite{resource}, \cite{llp}
The linear implications $X\multimap Y$ in the category are defined as

$X\multimap Y = X^{\bot}\parr Y$

\noindent since the category is symmetric monoidal closed. The definition of the categorical construction of the operation $\parr$ which is dual to tensor $\otimes$ is not discussed here for simplicity because this is not essential for our description. 
\end{definition}

A Conway game \emph{X} with a payoff is the game with an additional weight $k_{X} = \{1, 1/2, 0\}$ in each vertex \cite{resource}.  The weight depends on whether the position is winning or not. In the tensor production and implication, these weights obey rules of Boolean conjunction and implication. Thus, the payoff Conway game $X\otimes Y$ is defined as the underlying Conway game $X\otimes Y$, equipped with the payoff function $k_{X\otimes Y}(x\otimes y) = k_{X}(x)\wedge k_{Y}(y)$ and the payoff Conway game $X\multimap Y$ is defined as the underlying Conway game $X\multimap Y$, equipped with the payoff function $k_{X\multimap Y}(x \multimap y) = k_{X}(x)\Rightarrow k_{Y}(y)$. A strategy $\sigma$ on a payoff Conway game \emph{X} is winning when every play $s : x \mapsto y$ in the strategy ends in a winning position $y$, i.e., in a position of payoff 1/2 or 1.

It is possible to prove that the categorical construction is conserved if the weights' numbers are replaced with some sets which form a Brouwer lattice, and the Boolean operations are replaced with lattice operations. The greater set is connected with a position, the more advantageous it is. We suppose the existence of a universal set containing all the others. Thus, all such estimation sets form a complete lattice.
\begin{definition}
We call a strategy $\sigma$ on a payoff Conway game \emph{X} with position estimations in a lattice as \emph{winning} if every play $s : x \mapsto y$ in the strategy ends in a position $y$ of payoff in the lattice which is different from 0.
\end{definition}
We show now that such defined winning strategies do compose as in \cite{resource}.
\begin{proposition}\label{prop1}
The strategy $\rho\circ\sigma : X \multimap Z$  is winning when the two strategies $\sigma : X \multimap Y$ and $\rho : Y \multimap Z$ are winning.
\end{proposition}
\begin{proof}
It is known that strategies do compose \cite{resource}. Thus, it is sufficient to check the winning condition. We should to observe that the composition of two winning positions is winning:

$\begin{aligned}k_{X}(x)\Rightarrow k_{Y}(y)>0, \;i.e.\; x\multimap y \;is\; winning;\\
                     k_{Y}(y)\Rightarrow k_{Z}(z)>0, \;i.e.\; y\multimap z \;is \;winning;\\
                     implies\; k_{X}(x)\Rightarrow k_{Z}(z)>0, \;i.e.\; x\multimap z \;is\; winning.\end{aligned}$

\noindent Indeed, by definition, in $a\Rightarrow b = c$, \emph{c} is such a maximal lattice element which has the same meet with \emph{a} as \emph{b}. If $k_{X}(x)$, $k_{Y}(y)$, $k_{Z}(z)$ have a common meet different from 0, then evidently $k_{X}(x)\Rightarrow k_{Z}(z)>0$. If $k_{X}(x)\wedge k_{Y}(y) = a$, $k_{Y}(y)\wedge k_{Z}(z) = b$ and $a\wedge b =0$, then

$\left\{\begin{aligned}k_{X}(x)\Rightarrow k_{Z}(z) >0, \; k_{X}(x)\wedge k_{Z}(z) >0 \\
k_{X}(x)\Rightarrow k_{Z}(z) = \neg k_{X}(x) = k_{X}(x)\Rightarrow 0, \;k_{X}(x)\wedge k_{Z}(z) = 0.\end{aligned}\right.$

\noindent $\neg k_{X}(x) = 0$ when there exist such a lattice element $d\neq 0$ that $d = \wedge x_{i}$ for all lattice elements $x_{i}\neq 0$ and $k_{X}(x) \neq 1$. But this assumption contradicts the condition $k_{X}(x)\wedge k_{Z}(z) = 0$. If $k_{X}(x) = 1$ then $k_{X}(x)\Rightarrow k_{Z}(z) = k_{Z}(z)$. Thus, $k_{X}(x)\Rightarrow k_{Z}(z)>0$ in all cases.
\end{proof}
\begin{proposition}\label{prop2}
The category \textbf{SetPayoff} whose objects are payoff Conway games with position weights take values in a lattice, and whose morphisms $X \rightarrow Y$ are winning strategies in $X \multimap Y$ is symmetric monoidal closed.
\end{proposition}
\begin{proof}
The category of Conway games is symmetric monoidal close \cite{resource}. Therefore, it is sufficient to check that $(k_{X}(x)\wedge k_{Y}(y))\Rightarrow k_{Z}(z) = k_{X}(x)\Rightarrow (k_{Y}(y)\Rightarrow k_{Z}(z))$ for all positions. But this formula is valid in Heyting algebras (i.e. in Brouwer lattices).
\end{proof}

Thus, the symmetric monoidal closed categorical construction for payoff Conway games from \cite{resource} is conserved for lattice payoffs.

\section{System Behavior Description}
\label{sec:3}
We consider a cognition process in which a system investigates an environment.
\begin{Th}It is supposed that the system investigates the environment visible up to some horizon in each direction  (as if in the fog: the objects visible better are nearer to the system, Fig.~\ref{fig:2}) and builds images of the observed objects.
\end{Th}
\begin{figure}
% Use the relevant command to insert your figure file.
% For example, with the graphicx package use
  \includegraphics[scale=0.6]{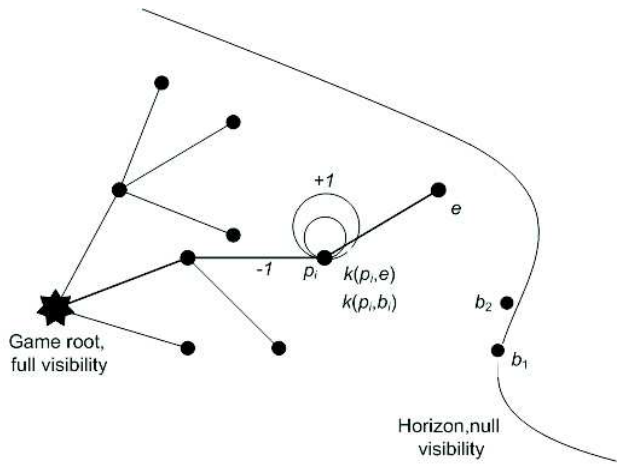}
% figure caption is below the figure
\caption{An example of a game and its visibility horizon. The bold line shows the resulting play.
The environment moves are depicted by circles.}
\label{fig:2}       % Give a unique label
\end{figure}
If the environment is a space of ideas, we also understand some its objects better than some others. Hence, those objects which we know better are located nearer to us in this space.

\begin{Th}It is supposed that in the case of ``\emph{tabula rasa}'' when the system has no images in the memory, an object attractive degree guides the system to behave: it investigates objects in the environment which have attracted its attention, and builds their images.
\end{Th}
The theory of attention is not discussed here (a review may be found in  \cite{Cog_Sci}).
The system should distribute the objects by the attraction degree during the investigation. Thus, some objects are more attractive, some of them are less attractive, and some cannot be compared by the attraction degree. Therefore, we get a partial order at the object set. It is supposed, that the set has the bottom element, i.e., the element of null interest, and the top element, i.e., the object of the greatest interest which is the most attractive element. The latter one is the lattice join of all its elements, i.e., the most attractive ``object'' is the join or combination of the lattice generators which are the individual objects or subobjects. All the others are their joins (combinations) or meets in the case when an attractive object is some part of the other object. Therefore, the system builds a complete lattice of the environment object images, their combinations, and allotments. Such a lattice may just exist; thus the attractions of new objects may be combined with the attractions of preexisting ones. Many simple biological systems, e.g., ants, possibly have such a preconstructed lattice of objects they can have a deal with. Perhaps the partial order in the lattice may be dynamically changed for more complicated systems.

The system may have a similar lattice of its goals or tasks which it can fulfill~\cite{Maximov}. As above, the goals or tasks may be combined or intersected to subitems. The higher goal lies in the lattice diagram (thus, the more tasks it includes), the more important this behavior variant is.
\begin{Th}It supposed that the preferable behavior of the system is to achieve all its goals.
\end{Th}
This variant corresponds to the top lattice element 1. In the case of environment images as the lattice elements the top element is the most attractive object. And the bottom element 0 corresponds to complete inactivity and to the least significant behavior variant. All the estimations may be considered as partially true truth values. Thus, we can say that the more essential the behavior is, the truer it is.

The lattice may also have an additional structure at it, e.g., the structure of linear logic~\cite{Maximov_17}, ~\cite{Maximov_MLSD13}. It is the latter case that is investigated in the present paper:
\begin{Th}the system has a preconstructed lattice of environment objects or the lattice of the system goals in the environment, and the lattice is provided with a monoid, i.e., with a linear logic structure.
\end{Th}
An example of such a goal/object lattice $M_{s}$ which may match the system is depicted in Fig. \ref{fig:1}. In the lattice, vertices $x_{i}$ and \emph{e} are the generators and denote the individual system goals or the environment objects. Vertices $J_{i,...}$ are the generators' joins and denote the combined goals' achievement. Vertices may have meets, i.e., some object part may be included in different objects (or correspondent goals). A variant of open (\emph{Op}) and closed (\emph{Cl}) classes definitions are indicated in the system lattice (Fig. \ref{fig:1}). This definition (with the multiplication definition) defines the structure of linear logic in the lattice. The element multiplication is obtained (usually ambiguously, \cite{Maximov_17}) from the demand to implement the linear logic operations properties. In this construction, it is possible to consider parallel processes of combined goals achieving as the tensor product of corresponding lattice elements in the logic (Sec. \ref{subsec:2}, the end). And the priority of different processes is obtained from the demand of the highest correspondent tensor product estimation in the lattice (the product is an element of the lattice, thus the higher it lies, the more important it is, i.e., the greater the truth-value of the correspondent process is).

The process of the system move in the environment space, i.e., the system's cognition process can be represented as a game in which the environment informs the system about (partially) visible objects at each step with the position reward which estimates the information. The reward is considered as a set which corresponds to the coincidence degree of the object and its image. Such coincidence degree sets create a lattice since they may have partial intersections (i.e., meets), and they may be combined (i.e., they have joins). Obviously, they have the full coincidence degree as the top and null coincidence degree as the bottom. Thus, we consider different lattices to estimate the cognition process --- the goal lattice and coincidence degree lattices for all goals.

In the case of ``\emph{tabula rasa}'' the total information about an object increases instead of the coincidence degree with the preexisting image. This information increasing may be represented as the topology refinement on the set which depicts the object image. The topology is also a lattice of the set subsets.

The image theory is not considered here, and its review may also be found in  \cite{Cog_Sci}. We regard the environment as the Proponent and the system as the Opponent because we need the negative game to use the categorical construction of Sec.~\ref{subsec:3}. The Opponent moves from one position in the environment to the other by the use of the information to achieve his goals, i.e., environments objects. The more fortunate the position, i.e., the larger the reward, the more precise information the environment provides about the object in the position. Thus, the completely winning position is the last game step of the Proponent in which the investigating object coincides with its image. And the system is placed initially in the configuration space (environment) in the root * of the system game \emph{A} with the system goal/object lattice $M_{s}$.

The game \emph{A} represents \emph{\textbf{possible}} system moves in the environment. But the \emph{\textbf{real}} trajectory or the play is chosen from the demand of the maximal total position reward along the projected path. The system move in the environment is estimated corresponding to an optimality criterion with the reward $k(p_{i},b_{j})$ in the position $p_{i}$ of the goal $b_{j}$ achieving process.

It may be that the system does not see any goal initially and moves according to a criterion of an optimal move in objects' absence. Thus, the system has, besides, a goal (task) $a$ of the movement in the environment, which, therefore, should be included in the lattice $M_{s}$. Hence, such a free movement may have its own value, though the system may simply wander and may not want to cognise anything.

The \emph{\textbf{optimality criterion}} may represent the highest degree of the correspondence to the demanding system configuration, or the greatest freedom in future moves or the most excellent visibility from a position or so on in the case of the free movement in the goals absence. Also, in the case of the system goals achieving we suppose the better a goal object is visible, the higher the reward is.

Let us \emph{n} goals $b_{1} ... b_{n}$ are discovered in the environment by the system with information about them  $k(p_{i},b_{j})$ in positions $p_{i}$ of the game \emph{A} Fig. \ref{fig:2}. The game \emph{A} corresponds to the process of achieving the free movement goal \emph{a}. Then, a winning strategy of the game $A'=A\multimap B_{1} \otimes ... \otimes B_{k}$ defines a transition (morphism) from the game \emph{A} to this new game $A'$ of moving and achieving goals $b_{1} ... b_{k}$  in the games $B_{1}  ...  B_{k}$. It supposed that the system can achieve several goals in parallel up to some moment when only one object should be chosen. Thus, the game $A'$ corresponds to parallel processes of achieving those \emph{k} goals from discovered \emph{n} ones, which may be better achieved in the next sense.

It is reasonable to choose the trajectory (play) from the demand to maximize the reward along the path within the visibility horizon\footnote{Strictly, this formula should look as follows \begin{equation}k^{A\multimap B_{1} \otimes ... \otimes B_{k}}_{play} = \max_{plays}\bigcup_{play}[k^{A}\Rightarrow (k^{B_{1}}\& ... \& k^{B_{k}})]
\end{equation} But this form is inconvenient in practice. Thus, we use (\ref{eqmain}) for which the correspondent Propositions \ref{prop1}, \ref{prop2} may be also proved.}:
\begin{equation}\label{eqmain}
k^{A\multimap B_{1} \otimes ... \otimes B_{k}}_{play} \equiv  k^{A^{\bot}\parr B_{1} \otimes ... \otimes B_{k}}_{play} = \max_{plays}[\bigcup_{play}k^{A^{\bot}}\bigcup_{play}(k^{B_{1}}\& ... \& k^{B_{k}})]
\end{equation}
Here the reward $k^{A\multimap B_{1} \otimes ... \otimes B_{k}}_{play} = k^{A^{\bot}\parr B_{1} \otimes ... \otimes B_{k}}_{play}$ is maximized in the game $A'$ and corresponds to that process of \emph{k} goals achieving that has the greatest priority
\begin{equation}\label{eq2}
a\multimap b_{1} \times ... \times b_{k} = (a\times (b_{1} \times ... \times b_{k})^{\bot})^{\bot}
\end{equation}
in the system goal lattice. Thus, there are \emph{k} the most important parallel processes from the viewpoint of the system goal lattice. The priority is maximal among all possible parallel processes of achieving \emph{n} discovered goals.

The maximum in (\ref{eqmain}) is taken among all possible plays, and it joins the rewards along these plays (i.e., trajectories) in games $A^{\bot}$, $B_{1}$, ... and $B_{k}$. Thus, it is demanded to maximize the coincidence degrees of all objects $b_{i}$ and their system images along the resultant path up to the moment when all the images together may not be able to improve. Then, the total number of chosen objects should be decreased by the same method: the system should pick those $l$ objects with $l < k$ which parallel achieving processes have the greatest estimation in the system goal lattice. And so on up to the moment, when only one goal remains.

In (\ref{eq2}) the multiplication $b_{1} \times ... \times b_{k}$ and, therefore, all the $b_{i}$ must be facts. But in the game category \textbf{SetPayoff} the objects are not distinguished by classes. Thus, every element in the goal lattice may correspond to a game, and we have to extend the tensor product definition in the lattice phase model to arbitrary elements following \textbf{SetPayoff} construction. This means that the tensor $\times$ is the monoidal multiplication $\cdot$ in reality.

%But the tensor $b_{1} \times ... \times b_{k}$ in linear implication in (\ref{eq2}) should still be a fact \cite{Maximov_17}, \cite{Maximov_18}.

%and in (\ref{eq2}) we change $(a\times (b_{1} \times ... \times b_{k})^{\bot})^{\bot} \mapsto a^{\bot}\times b_{1} \times ... \times b_{k}$ as in \textbf{SetPayoff}.

If there are parallel processes which are not compared by priority, i.e., if there are several incompatible priorities $(a\times (b_{1} \times ... \times b_{k})^{\bot})^{\bot}$, it is possible to reorder the goal lattice in the manner that some lattice vertex would be used as an additional priority \cite{Maximov}. In the reordered lattice, initially incompatible elements can become compatible so we can choose the preference of the processes. Such a lexicographical rule follows from the teleological system assignment, i.e., our vision of the system purpose. It may be that the lattice is so symmetric that there is no any reason to distinguish the generator priorities in a non-trivial manner in general. In this case, we may use the fact that there is an ambiguity in the linear logic structure definition \cite{Maximov_17}, \cite{Maximov_18}. The restriction of the ambiguity may explain the behavior of a concrete system.

\subsection{An example of the ambiguity restriction}\label{subsec:3.1}
Specifically, let us consider the lattice in Fig. \ref{fig:1} which is taken from \cite{Maximov_18}. In our interpretation, the vertex $a$ corresponds to the process of a free system move and vertexes $b_{i}$ and $e$ are the goals (i.e., the environment objects or task processes). The vertex $a$ is included in vertexes $J_{1a}, e, b_{2}, b_{3}$ because of the system should move in the environment to achieve the complete cognition of the goals (or tasks fulfilling). Strictly, all the vertexes $e, b_{2}, b_{3}$ should be represented as $J_{1a}$ as the join of a goal and its achieving process, but it is not done for simplicity. In \cite{Maximov_18}, the demand of linear logic properties validness defines the lattice elements' monoidal multiplications. But these properties are not enough to define the multiplications uniquely. Thus, we have an ambiguity which is signed by $\vee$ in the next multiplication definitions obtained in \cite{Maximov_18}: e.g., $b_{1}a = 0 \vee a$ means that the multiplication $b_{1}a$ may be as 0 as $a$.
\begin{multline}\label{syst}
\left\{\begin{aligned}
b_{1}b_{1} = b_{1}\\
b_{1}b_{2} = 0\\
b_{1}b_{3} = 0\\
b_{1}a = 0 \vee a\\
b_{2}a = a \vee 0\\
b_{1}a + b_{2}a = a\\
b_{1}e = 0\\
b_{2}b_{2} = b_{2}\\
b_{2}b_{3} = b_{3}\\
b_{2}e = e\\
b_{3}b_{3} = e \vee b_{2} \vee b_{3} \vee J_{23} \vee J_{2e} \vee J_{3e} \vee J_{23e}\\
b_{3}e = e \vee J_{2e} \vee J_{3e} \vee J_{23e}\\
b_{3}J_{3e} = e \vee J_{2e} \vee J_{3e} \vee J_{23e}\\
b_{3}a = b_{3} \vee a \vee 0\\
ae = J_{23e} \vee J_{2e} \vee J_{3e} \vee e\\
ee = e \vee J_{2e} \vee J_{3e} \vee J_{23e}
\end{aligned}\right.
\end{multline}
Let us the system has discovered all the goals $b_{i}, e$ but the object $e$ is seen better than the others for definiteness Fig. \ref{fig:2}. Therefore, it is more attractive than the others and it may be considered as an additional priority, so we should consider only those games which include the goal $e$ achieving. Thus, we are interested in the next possible game $A'$ estimation variants of achieving the discovered goals (the definitions of Sec. \ref{subsec:2} are used):
\begin{equation}\label{1}
a\multimap (J_{1a}\times e \times b_{2})= (a\times (J_{1a}\times e \times b_{2})^{\bot})^{\bot}
\end{equation}
\begin{equation}\label{2}
a\multimap (J_{1a}\times e \times b_{3})= (a\times (J_{1a}\times e \times b_{3})^{\bot})^{\bot}
\end{equation}
\begin{equation}\label{3}
a\multimap (J_{1a}\times e \times b_{2}\times b_{3})= (a\times (J_{1a}\times e \times b_{2}\times b_{3})^{\bot})^{\bot}
\end{equation}
\begin{equation}\label{5}
a\multimap ( e \times b_{2}\times b_{3})= (a\times (e \times b_{2}\times b_{3})^{\bot})^{\bot}
\end{equation}
\begin{equation}\label{6}
a\multimap (e \times b_{2})= (a\times (e \times b_{2})^{\bot})^{\bot}
\end{equation}
\begin{equation}\label{7}
a\multimap (e \times b_{3})= (a\times (e \times b_{3})^{\bot})^{\bot}
\end{equation}
\begin{equation}\label{8}
a\multimap (J_{1a}\times e)= (a\times (J_{1a}\times e)^{\bot})^{\bot}
\end{equation}
\begin{equation}\label{9}
a\multimap (e)= (a\times (e)^{\bot})^{\bot}
\end{equation}

It is seen from (\ref{syst}) that it should be $b_{1}a = 0$ and $b_{2}a = a$ or $b_{1}a = a$ and $b_{2}a = 0$. In the latter case, all the estimations (\ref{1})--(\ref{9}) are equal to $J_{123}$ and, therefore, are not distinguishable. Thus, the system should be pre-constructed in the way as $b_{1}a = 0$ and $b_{2}a = a$ to be able to distinguish the behavior variants.

In this case, we obtain the following estimations:
\begin{equation}\label{1}
a\multimap (J_{1a}\times e \times b_{2})= 1
\end{equation}
\begin{equation}\label{2}
a\multimap (J_{1a}\times e \times b_{3})= \left\{\begin{aligned}1, \;if\; ab_{3}=a\vee 0\\
                                                                b_{1},\; if\; ab_{3}=b_{3}
                                                  \end{aligned}\right.
\end{equation}
\begin{equation}\label{3}
a\multimap (J_{1a}\times e \times b_{2}\times b_{3})= \left\{\begin{aligned}1, if ab_{3}=a\vee 0\\
                                                                b_{1}, if ab_{3}=b_{3}
                                                  \end{aligned}\right.
\end{equation}
\begin{equation}\label{5}
a\multimap ( e \times b_{2}\times b_{3})= 1
\end{equation}
\begin{equation}\label{6}
a\multimap (e \times b_{2})= 1
\end{equation}
\begin{equation}\label{7}
a\multimap (e \times b_{3})= 1
\end{equation}
\begin{equation}\label{8}
a\multimap (J_{1a}\times e)= 1
\end{equation}
\begin{equation}\label{9}
a\multimap (e)= 1
\end{equation}
Also, it is seen that the system should be pre-constructed as $ab_{3}=b_{3}$ to be able to distinguish the behavior variants. Then, we see that variants (\ref{1}) and (\ref{5}) are preferable because they have the maximal estimations and include the greatest possible goals' number. We suppose that the behavior variant which includes more goals achieving is preferable for the system with other equal estimations. To distinguish the two remain variants, the system may have some additional reasons to consider one or the other goal as more attractive.

\subsection{Comments}
The additional fine point is that the graphs of the games \emph{A} and \emph{B} are not always coincide, i.e., a goal may be seen in \emph{B}, but the way to the goal may not exist in \emph{A} (i.e., in the environment). Thus, we have a deal with the game \emph{A} of the system moves (and move rewards), on the positions of which the rewards of \emph{k} games \emph{B} are possible but are not obligatory.

It should be also pointed out that the structures of tensorial
multiplication in $(a\times (b_{1} \times ... \times b_{k})^{\bot})^{\bot}$ and in $A^{\bot}\parr B_{1} \otimes ... \otimes B_{k}$ are different. In
$(a\times (b_{1} \times ... \times b_{k})^{\bot})^{\bot}$ the tensor $\times$ is a monoid multiplication in the goal lattice. The structure is pre-existent and does not depend on the environment. It may be chosen for very general considerations \cite{Maximov_17}, \cite{Maximov_18}, and is determined by the system purpose. But the tensor $\otimes$ (and, correspondent, the co-tensor $\parr$) in $A^{\bot}\parr B_{1} \otimes ... \otimes B_{k}$ are defined in the monoidal closed category and are determined by the environment, by objects and obstacles distributions, by visibility and so on. Thus, there are two different linear logic structures in the approach. Obviously, these two structures should match each other for the system to function successfully. The variability of ways to chose the monoid provides perhaps the survival of the most successful systems in nature. But the way to chose the suitable monoid structure in advance is not clear.

\section{Conclusion}
\label{concl}
In the paper, a process of an environment cognition by an intelligent system was considered as a movement in the environment. The system move was represented as a game in which the environment corresponds to the Proponent which provides the Opponent (the system) with some information about the environment objects which may also be considered as the system goals. Different goals achieving is considered as parallel processes which are represented as the tensor product of correspondent games and form a comprehensive game.

The game has rewards on its positions, which estimates the quality of the information provided by the environment. The reward may be very different: it may represent the highest freedom in future moves, or the most excellent visibility from a position, or the degree of the coincidence of the goal object image with the original or some else. We demand the greatest total reward along the system play to choose the path. When the total reward of all parallel processes cannot be improved more, we decrease the number of selected goal achieving processes up to the one in the end.

We choose those goal achieving processes from all possible, which have the highest estimation in the system goal lattice. It is so because every goal and the corresponding process of its achieving has a definite correspondent truth value in the lattice. The higher the value lies in the lattice diagram, the higher priority of the process is. When it is not possible to determine the most important estimation, some additional methods may be used to select the optimal variant.

Thus, we consider two types of estimations: the goal lattice value defines the choice of the goal achieving processes from all possible ones, and the position rewards of the game determine the optimal path of these chosen processes in the environment. Different linear logic structures give both of the estimations --- on the goal lattice and the game category, i.e., in the environment. Thus, we may conclude that the linear logic structures play an essential role in the behavior determination of intelligent systems.

Such a model corresponds to the approach in which the system intelligence is considered as the consequence of the system predestination, i.e., of the structures on the system goal set or on the set of tasks which the system can fulfill. In simple systems, i.e., ants, such structures may be pre-existent. In more complicated ones, these structures may be being built and changed during the system life.

%\begin{acknowledgements}
%If you'd like to thank anyone, place your comments here
%and remove the percent signs.
%\end{acknowledgements}

% BibTeX users please use one of
\bibliographystyle{spbasic}      % basic style, author-year citations
%\bibliographystyle{spmpsci}      % mathematics and physical sciences
%\bibliographystyle{spphys}       % APS-like style for physics
%\bibliography{Maximov_Cognition}   % name your BibTeX data base

\begin{thebibliography}{14}
\providecommand{\natexlab}[1]{#1}
\providecommand{\url}[1]{{#1}}
\providecommand{\urlprefix}{URL }
\expandafter\ifx\csname urlstyle\endcsname\relax
  \providecommand{\doi}[1]{DOI~\discretionary{}{}{}#1}\else
  \providecommand{\doi}{DOI~\discretionary{}{}{}\begingroup
  \urlstyle{rm}\Url}\fi
\providecommand{\eprint}[2][]{\url{#2}}

\bibitem[{Birkhoff(1967)}]{Birkhoff}
Birkhoff G (1967) Lattice Theory. Providence, Rhode Island

\bibitem[{Cassimatis(2012)}]{Cass}
Cassimatis NL (2012) Artificial intelligence and cognitive modeling have the
  same problem. In: Theoretical Foundations of Artificial General Intelligence,
  Atlantis Press, pp 11--24

\bibitem[{Everitt and Hutter(2018)}]{Hutt18}
Everitt T, Hutter M (2018) Universal artificial intelligence. In: Abbass H,
  Scholz J, Reid D (eds) Foundations of Trusted Autonomy. Studies in Systems,
  Decision and Control, Springer, Cham, vol 117, pp 15--46

\bibitem[{Friedenberg and Silverman(2006)}]{Cog_Sci}
Friedenberg J, Silverman G (2006) Cognitive science: an introduction to the
  study of mind. Sage Publications, Thousand Oaks – London – New Delhi

\bibitem[{Girard(1987)}]{jerar1987}
Girard JY (1987) Linear logic. Theoretical Computer Science (50):1–--102

\bibitem[{Hutter(2005)}]{HuttB}
Hutter M (2005) Universal Artificial Intelligence: Sequential Decisions based
  on Algorithmic Probability. Springer, Berlin

\bibitem[{Hutter(2012)}]{Hutt}
Hutter M (2012) One decade of universal artificial intelligence. In:
  Theoretical Foundations of Artificial General Intelligence, Atlantis Press,
  pp 66--88

\bibitem[{Lafont(1999-2017)}]{llp}
Lafont Y (1999-2017) Linear logic pages. Tech. rep.,
  \urlprefix\url{http://iml.univ-mrs.fr/~lafont/pub/llpages.pdf}

\bibitem[{Legovich and Maximov(2017)}]{UBS}
Legovich YS, Maximov DY (2017) Selecting executor in a group of intellectual
  agents. Automation and Remote Control 78(7):1341--1349

\bibitem[{Maksimov(2016)}]{Maximov}
Maksimov DY (2016) Reconfiguring system hierarchies with multi--valued logic.
  Automation and Remote Control 77(3):462--472

\bibitem[{Maximov and Ryvkin(2017)}]{Maximov_MLSD13}
Maximov D, Ryvkin S (2017) Systems smart effects as the consequence of the
  systems complexity. In: Proc. 17th International Conf. on Smart Technologies
  (IEEE EUROCON 2017, Ohrid), IEEE, Ohrid, pp 576--582

\bibitem[{Maximov and et~al(2017)}]{Maximov_17}
Maximov DY, et~al (2017) How the structure of system problems influences system
  behavior. Automation and Remote Control 78(4):689--699

\bibitem[{Maximov and et~al(2018)}]{Maximov_18}
Maximov DY, et~al (2018) A mixed group of manned and unmanned aerial vehicles
  control. unpublished

\bibitem[{Mellies and Tabareau(2010)}]{resource}
Mellies PA, Tabareau N (2010) Resource modalities in tensor logic. Ann Pure
  Appl Logic 161(5):632--653

\end{thebibliography}

% Non-BibTeX users please use

\end{document}